\documentclass{article}

\usepackage{microtype}
\usepackage{graphicx}
\usepackage{subcaption}
\usepackage{booktabs} 

\usepackage{hyperref}

 \usepackage[preprint]{icml2026}

\usepackage{amsmath}
\usepackage{amssymb}
\usepackage{mathtools}
\usepackage{amsthm}
\usepackage{todonotes}

\usepackage[capitalize,noabbrev]{cleveref}

\theoremstyle{plain}
\newtheorem{example}{Example}[section]
\newtheorem{theorem}{Theorem}[section]
\newtheorem{proposition}[theorem]{Proposition}
\newtheorem{lemma}[theorem]{Lemma}

\theoremstyle{definition}
\newtheorem{definition}[theorem]{Definition}

\theoremstyle{remark}

\newcommand{\pref}{\textit{Pref}}

\icmltitlerunning{Language Generation in the Limit: 
  Complexity Barriers and Implications for Learning}

\begin{document}

\twocolumn[
  \icmltitle{Language Generation in the Limit: \\ 
  Complexity Barriers and Implications for Learning}



  \icmlsetsymbol{equal}{*}

  \begin{icmlauthorlist}
    \icmlauthor{Marcelo Arenas}{dcc,imfd}
    \icmlauthor{Pablo Barcel\'{o}}{imfd,imc,cenia}
    \icmlauthor{Luis Cofr\'{e}}{mat}
    \icmlauthor{Alexander Kozachinskiy}{cenia}
  \end{icmlauthorlist}

  \icmlaffiliation{dcc}{Department of Computer Science, Pontifical Catholic University of Chile}
     \icmlaffiliation{imfd}{Instituto Milenio Fundamentos de los Datos}
        \icmlaffiliation{imc}{Institute for Mathematical and Computation Engineering, Pontifical Catholic University of Chile}
   	 \icmlaffiliation{mat}{Faculty of Mathematics, Pontifical Catholic University of Chile}
                \icmlaffiliation{cenia}{National Center for Artificial Intelligence of Chile}

  \icmlcorrespondingauthor{Marcelo Aranas}{marenas@uc.cl}
  \icmlcorrespondingauthor{Pablo Barcel\'{o}}{pbarcelo@uc.cl}
  \icmlcorrespondingauthor{Luis Cofr\'{e}}{luis.cofr@uc.cl}
  \icmlcorrespondingauthor{Alexander Kozachinskiy}{alexander.kozachinskyi@cenia.cl}

  \vskip 0.3in
]



\printAffiliationsAndNotice{}  

\begin{abstract}
  Kleinberg and Mullainathan showed that language generation in the limit is always possible
at the level of computability: given enough positive examples, a learner can eventually
generate data indistinguishable from a target language. However, such existence results do
not address feasibility.
We study the sample complexity of language generation in the limit for several canonical
classes of formal languages. Our results show that infeasibility already appears for
context-free and regular languages, and persists even for strict subclasses such as locally threshold
testable languages, as well as for incomparable classes such as non-erasing pattern
languages, a well-studied class in the theory of language identification.
Overall, our results establish a clear gap between the theoretical possibility of
language generation in the limit and its computational feasibility.
\end{abstract}

\section{Introduction}

\paragraph{Context.}
The study of language identification in the limit originates in the seminal work of
\citet{DBLP:journals/iandc/Gold67}. Informally, a collection $\mathcal{L}$ of languages is
said to be \emph{identifiable in the limit} if there exists a learning algorithm that,
given an infinite stream of positive examples from some target language
$L \in \mathcal{L}$, produces a sequence of hypotheses that eventually stabilizes on a
correct description of $L$. Gold showed that this notion is highly restrictive: even the
class of regular languages is not identifiable in the limit from positive data alone,
whereas finite families trivially are. Angluin subsequently refined this negative picture
by providing a structural characterization of families that are identifiable in the
limit~\cite{DBLP:journals/iandc/Angluin80}. Despite these refinements, few natural language
families satisfy these conditions. As a result, Gold--Angluin style unlearnability results
have often been interpreted as formal support for the \emph{poverty of the stimulus}
hypothesis in linguistics, suggesting that purely data-driven learning is insufficient
without strong innate biases or prior structure
\cite{pinker2013learnability,Pearl2021PoS}.

A notable exception to this pessimistic landscape is provided by \emph{non-erasing pattern
languages}~\cite{DBLP:journals/jcss/Angluin80}. These languages are generated by words over
constants and variables, where each variable must be instantiated by a nonempty word.
Pattern languages are expressive enough to capture nontrivial structural regularities,
yet sufficiently constrained to form a canonical example of a nontrivial family that is
learnable in Gold’s model~\cite{DBLP:journals/iandc/Angluin80}.

The recent empirical success of large language models (LLMs) contrasts with these
classical learnability limitations. Although trained on finite samples and exposed
primarily to positive data, LLMs appear to internalize a wide range of syntactic and
semantic regularities, suggesting that useful generalization may be possible without
explicitly identifying an underlying grammar. Motivated by this tension, Kleinberg and
Mullainathan introduced the notion of \emph{language generation in the
limit}~\cite{kleinberg2024language}, which shifts the focus from identifying a language to
generating strings of a target language, unseen in the training data. Unlike
identification, language generation in the limit is always possible at the level for countable families of languages. This establishes a
fundamental asymmetry between learning and generation and motivates closer study of
generation as a lens for understanding generalization. 

\paragraph{Problem.}
While language generation in the limit is computable in full generality, this fact alone
says little about \emph{feasibility}. From a computational standpoint, the key question is
how many examples are required before a generator can reliably produce strings
indistinguishable from those of a target language. If this sample complexity grows too
rapidly with the size or structure of the language family, then generation in the limit
remains a purely existential result.

\paragraph{Our results.}
We study the sample complexity of language generation in the limit for finite families of
languages drawn from several canonical language classes. While non-computable bounds are
easy to obtain for highly pathological families, we show that severe feasibility barriers
already arise for central and well-studied classes.
\begin{itemize}
\item {\em Context-free languages:}
We show that there is no computable bound on the number of examples required to guarantee
successful generation in the limit, even for families specified by pushdown automata or
context-free grammars. Thus, non-computability already arises for a canonical and
extensively studied class.

\item {\em Regular languages:}
Restricting attention to regular languages restores computability, but at a high cost. For
finite families of regular languages represented by (deterministic) finite automata or
regular expressions, any generator requires a double-exponential number of examples in
the size of the family.

\item {\em LTT languages:}
We next consider \emph{locally threshold testable} (LTT) languages, a robust and
well-studied subclass of the regular languages. Informally, LTT languages are determined
by local counting conditions on short substrings and admit equivalent combinatorial,
logical, and algebraic characterizations. For this class, the sample complexity of
generation in the limit is precisely single exponential. While this improves over the
double-exponential bound for arbitrary regular languages, the resulting complexity
remains infeasible.

\item {\em Patterns:}
Finally, we study the previously mentioned non-erasing pattern languages, which are
incomparable with the regular and the LTT languages. Despite their favorable properties for
identification in the limit, we establish an exponential lower bound on the sample
complexity of generation over finite families.
\end{itemize}

\section{Language generation}

\paragraph{Generation in the limit.}
Let $\Sigma$ be a finite alphabet and let $\Sigma^*$ denote the set of all finite words
over $\Sigma$. A \emph{language} over $\Sigma$ is a subset of $\Sigma^*$.
A \emph{generator} is a function $G$ mapping a finite set of words from $\Sigma^*$ to a
word in $\Sigma^*$.

\begin{definition}[Language generation in the limit]
A set $\mathcal{F}$ of infinite languages over $\Sigma$ is \emph{generatable in the limit}
if there is a generator $G$ such that, for every $L \in \mathcal{F}$ and  ordering
$w_1,w_2,\ldots$ of the words in $L$, there is $m \ge 1$ with
$G(\{w_1,\ldots,w_n\}) \in L \setminus \{w_1,\ldots,w_n\}$ for all $n \ge m$.
\qed
\end{definition}

Language generation in the limit was introduced by
\citet{kleinberg2024language}, who showed that every countable family of languages is
generatable in the limit. However, this notion is inherently non-quantitative: the point
of convergence may depend arbitrarily on the target language and on the ordering of its
examples, and therefore does not yield a meaningful notion of sample complexity.

\paragraph{Uniform generation.}
To reason about feasibility, one needs guarantees that hold after a fixed number of
examples, uniformly across the family. This motivates the notion of
\emph{uniform generation}, also introduced in \cite{kleinberg2024language}. In particular,
they observe that every finite family of languages admits such a uniform bound, making
uniform generation the appropriate setting for studying generation from finite collections
of languages.

\begin{definition}[Uniform generation]
Let $m \ge 1$. A set $\mathcal{F}$ of infinite languages is \emph{$m$-generatable} if there
exists a generator $G$ such that, for every $L \in \mathcal{F}$ and every finite
$X \subseteq L$ with $|X| \ge m$, it holds that $G(X) \in L \setminus X$.

Such a generator $G$ is called an \emph{$m$-generator} for $\mathcal{F}$.
\qed
\end{definition}

Families that are $m$-generatable for some $m \in \mathbb{N}$ were characterized by
\citet{li2024generation}. We recall a streamlined version of this characterization,
adapted to our setting.

\begin{proposition}
\label{prop_char}
A set $\mathcal{F}$ of infinite languages is not $m$-generatable if and only if there
exists a non-empty $\mathcal{S} \subseteq \mathcal{F}$ such that
$\bigcap_{L \in \mathcal{S}} L$ is finite and has size at least $m$.
\end{proposition}

\begin{proof}
Suppose there exists a non-empty $\mathcal{S}\subseteq \mathcal{F}$ such that
$X = \bigcap_{L\in\mathcal{S}} L$ is finite and $|X|\ge m$.
Assume for contradiction that an $m$-generator $G$ exists.
If $G(X)\in X$, then for any $L\in\mathcal{S}$ we have $G(X)\in L$, contradicting
$G(X)\in L\setminus X$.
If $G(X)\notin X$, then $G(X)\notin L$ for some $L\in\mathcal{S}$, again a contradiction.

Conversely, assume no such $\mathcal{S}$ exists.
For a finite set $X\subseteq\Sigma^*$, define
$\mathcal{S}_X=\{L\in\mathcal{F}\mid X\subseteq L\}$ and
$A_X=\bigcap_{L\in\mathcal{S}_X}L$.
Define $G(X)$ to be any element of $A_X\setminus X$ when $\mathcal{S}_X$ is non-empty.
If $X\subseteq L'$ with $|X|\ge m$, then $A_X\subseteq L'$ has size at least $m$ and,
by assumption, must be infinite; hence $A_X\setminus X\neq\emptyset$ and
$G(X)\in L'\setminus X$.
\end{proof}

When $\mathcal{F}$ is finite, it is necessarily $m$-generatable for some $m$, since only
finitely many intersections $\bigcap_{L\in\mathcal{S}}L$ can arise.
However, this observation concerns existence only.
To study feasibility, one must quantify how large $m$ must be as a function of the size
of $\mathcal{F}$ and the descriptive complexity of its languages.
Accordingly, in the remainder of the paper we focus on uniform generation for finite
families and analyze the sample complexity required for generation in the limit across
several canonical language classes.

\paragraph{Related work.}
Before turning to our technical results, we briefly review related work on language
generation. \citet{li2024generation} introduce the notion of
\emph{non-uniform generation}, where the number of examples required for generation may
depend on the target language, but not on the order in which examples are presented.
They distinguish this notion from uniform generation and generation in the limit, and
relate it to classical learning frameworks such as PAC and online learning.
Subsequently, \citet{DBLP:journals/corr/abs-2506-18642} study closure properties of these
generation notions under language union.

A different line of work focuses on the \emph{breadth} of generation, that is, on how
representative the generated subset is of the target language.
\citet{DBLP:conf/stoc/KalavasisMV25} show that one cannot, in general, converge to
generating all unseen elements of the target language, even under a statistical (rather
than adversarial) model of example presentation.
On the positive side, \citet{DBLP:journals/corr/abs-2504-14370} demonstrate that it is
possible to generate a subset of the target language with positive density.
Finally, \citet{DBLP:journals/corr/abs-2505-21819} propose a notion of
\emph{representative} language generation inspired by algorithmic fairness, requiring
that the relative proportions of different segments of the language in the generated
output approximately match those observed in the data.

These works address complementary aspects of language generation, such as robustness,
coverage, and representativeness. In contrast, our focus is on the \emph{sample
complexity} of uniform generation, and on understanding how it depends on the structural
properties of the underlying language class.

\section{Results for context-free languages}

We begin with context-free languages, a canonical and extensively studied class in formal
language theory. Our objective is to understand how many examples are required to generate
from an unknown context-free language drawn from a finite family. We show that, even in
this restricted setting, there is no computable bound on the number of examples sufficient
for uniform generation, already for families consisting of only two infinite languages.

Throughout this section, we assume familiarity with Turing machines and the
undecidability of the halting problem.

\subsection{Context-free grammars}

Context-free grammars are a standard and widely used formalism for specifying
context-free languages. We briefly recall their definition 
A \emph{context-free grammar} (CFG) is a tuple
$G = (\mathcal{X}, \Sigma, R, S)$,
where:
\begin{itemize}
\item $\mathcal{X}$ is a finite set of \emph{variables}, denoted by
$X,Y,\ldots$. 
\item $\Sigma$ is a finite alphabet of \emph{constants}, disjoint from $\mathcal{X}$ and denoted by $a,b,\ldots$. 
\item $R \subseteq \mathcal{X} \times (\mathcal{X} \cup \Sigma)^*$ is a finite set of
{\em production rules}. 
\item $S \in \mathcal{X}$ is the {\em start} variable.
\end{itemize}

We define the \emph{one-step derivation relation} $\Rightarrow_G \subseteq
(\mathcal{X} \cup \Sigma)^* \times (\mathcal{X} \cup \Sigma)^*$ as follows. For words
$u,v \in (\mathcal{X} \cup \Sigma)^*$, we write $u \Rightarrow_G v$ if there exist
words $x,y \in (\mathcal{X} \cup \Sigma)^*$, a variable $X \in \mathcal{X}$, and a rule
$(X \to \alpha) \in R$ such that
\[
u = x X y \quad\text{and}\quad v = x \alpha y.
\]

The relation $\Rightarrow_G^*$ denotes the reflexive and transitive closure of
$\Rightarrow_G$. A word $w \in (\mathcal{X} \cup \Sigma)^*$ is said to be \emph{derived}
from a variable $X \in \mathcal{X}$ if $X \Rightarrow_G^* w$.

The language generated by $G$ is
\[
L(G) \;=\; \{\, w \in \Sigma^* \mid S \Rightarrow_G^* w \,\}.
\]
A language is \emph{context-free} if it is generated by some context-free grammar.

\begin{example}
{\em Let $\Sigma = \{a,b\}$. 
The language $\{a^n b^n \mid n \ge 0\}$ is context-free. It is generated by the grammar
$G = (\{S\}, \{a,b\}, R, S)$ with production rule
$S \;\to\; aSb \;\mid\; \epsilon$, where $\epsilon$ is the empty word. 
In contrast, the language $\{a^n b^n c^n \mid n \ge 0\}$ is not context-free.
}
\qed
\end{example}

\subsection{Non-computable sample complexity}

We now state our main result for context-free languages.

\begin{theorem}
\label{thm_context}
There is no algorithm that, given two infinite CFGs
$G_1$ and $G_2$, outputs a number $m \in \mathbb{N}$ such that the family
$\{L(G_1), L(G_2)\}$ is $m$-generatable. The algorithm may diverge if either
$L(G_1)$ or $L(G_2)$ is finite.
\end{theorem}

\begin{proof}
By reduction from the halting problem.

\begin{lemma}
\label{lem_turing}
There exists an algorithm that, given a Turing machine $M$ with empty input, constructs
two CFGs $G_1$ and $G_2$ such that, if $M$ halts in exactly $t$
computation steps, then
\[
|L(G_1) \cap L(G_2)| = 2^t.
\]
\end{lemma}

Assuming Lemma~\ref{lem_turing}, we show that the existence of the algorithm postulated in
Theorem~\ref{thm_context} would imply decidability of the halting problem.
Given a Turing machine $M$, we construct CFGs $G_1$ and $G_2$ as in
Lemma~\ref{lem_turing}. It is a classical result that finiteness of context-free languages is decidable: given a
context-free grammar $G$, one can determine whether $L(G)$ is finite and, if so, compute its
cardinality (see, e.g.,~\cite{Ginsburg1966CFL}).

We distinguish three cases:

\begin{enumerate}
\item If $L(G_1)$ is finite, choose $t$ such that
$2^t > |L(G_1)| \ge |L(G_1) \cap L(G_2)|$. Then $M$ cannot halt in $t$ or more steps. We
simulate $M$ for $t-1$ steps and conclude that it does not halt if no halting configuration
is reached.

\item The case where $L(G_2)$ is finite is symmetric.

\item If both $L(G_1)$ and $L(G_2)$ are infinite, apply the algorithm from
Theorem~\ref{thm_context} to obtain $m$ such that
$\{L(G_1), L(G_2)\}$ is $m$-generatable. Choose $t$ with $m \le 2^t$.
By Proposition~\ref{prop_char}, if $M$ were to halt in $s \ge t$ steps, then
$|L(G_1) \cap L(G_2)| = 2^s \ge 2^t \ge m$, contradicting $m$-generatability. Hence $M$
cannot halt in $t$ steps, and it suffices to simulate $M$ for $t-1$ steps.
\end{enumerate}
Therefore, the assumed algorithm cannot exist.
\end{proof}

\begin{proof}[Proof of Lemma~\ref{lem_turing}]
Assume that $M$ halts in exactly $t$ computation steps, and let
$C_1, \ldots, C_t$ be the sequence of configurations of $M$ (encoded as words over a finite
alphabet in the standard way). We add two fresh delimiter constants
$\#_0$ and $\#_1$.
We construct CFGs 
$G_1$ and $G_2$ such that
$L(G_1) \cap L(G_2)$ consists precisely of the words
\[
C_1 \#_{i_1} C_2^R \#_{i_2} \cdots C_t^R \#_{i_t}
\]
if $t$ is even, and
\[
C_1 \#_{i_1} C_2^R \#_{i_2} \cdots C_t \#_{i_t}
\]
if $t$ is odd, where each $i_j \in \{0,1\}$ and $w^R$ denotes the reversal of $w$.

The CFGs enforce the following conditions:
\begin{itemize}
\item $C_1$ encodes the initial configuration of $M$;
\item $C_t$ encodes the unique accepting configuration;
\item for each $i < t$, the pair $(C_i, C_{i+1})$ represents a valid transition of $M$.
\end{itemize}

The verification of transition correctness is distributed between $G_1$ and $G_2$,
alternating between odd and even indices. Each grammar checks only local consistency
conditions between consecutive configurations, which can be enforced by context-free
productions using standard techniques.

Each choice of the delimiter sequence $(i_1,\ldots,i_t)$ yields a distinct word in
$L(G_1) \cap L(G_2)$, and all such words arise in this way. Hence,
$|L(G_1) \cap L(G_2)| = 2^t$.
\end{proof}

\subsection{Other descriptions for context-free languages}

The lower bound established in Theorem~\ref{thm_context} is robust with respect to the
choice of formalism used to specify context-free languages. In particular, the same
non-computable sample complexity arises when languages are given by \emph{pushdown
automata} (PDAs), that is, finite-state machines equipped with an unbounded stack, since
every CFG can be translated into an equivalent PDA, and vice versa, with
only polynomial overhead (see, e.g.,~\cite{HopcroftUllman2007}).
Moreover, the construction underlying the proof can be carried out using
\emph{deterministic} PDAs. Since deterministic PDAs form a strict subclass of the
context-free languages, this shows that the source of non-computability is not
nondeterminism or excessive expressive power, but already manifests under strong
determinism constraints.

\section{Results for regular languages} 

We now 
consider $m$-generatability for finite families of regular languages specified by
regular expressions or finite automata. We prove that, in the worst case, the sample complexity of uniform
generation is double exponential in the size of the family.

\subsection{Regular expressions and finite automata}

Regular expressions are a standard and widely used formalism for specifying
regular languages. We briefly recall their definition.
Let $\Sigma$ be a finite alphabet. 
The set of \emph{regular expressions} over $\Sigma$ is defined by the grammar
\[
r \;::=\; \emptyset \;\mid\; \varepsilon \;\mid\; a \;\mid\; (r + r) \;\mid\; (rr) \;\mid\; r^*,
\]
where $a$ ranges over $\Sigma$.
Each regular expression $r$ defines a language $L(r) \subseteq \Sigma^*$, where
$\emptyset$ defines $\emptyset$, $\varepsilon$ defines $\{\varepsilon\}$, each $a \in \Sigma$
defines $\{a\}$, $(r+s)$ defines $L(r)\cup L(s)$, $(rs)$ defines
$\{uv \mid u\in L(r),\, v\in L(s)\}$, and $r^*$ defines
$\bigcup_{k\ge 0} L(r^k)$, with $r^0=\varepsilon$ and $r^k = (r r^{k-1})$ for $k \geq 1$.

A language over $\Sigma$ is \emph{regular} if it is defined by some regular expression over
$\Sigma$.

\begin{example}
{\em
The regular language defined by $a^* b^*$ consists of all words over
$\{a,b\}$ in which no occurrence of $a$ appears after an occurrence of $b$. In contrast,
the context-free language $\{a^n b^n \mid n \ge 0\}$ is not regular.
}
\qed
\end{example}

A regular language can alternatively be specified by a finite automaton.
Recall that a {\em non-deterministic finite automaton (NFA)} over an alphabet $\Sigma$ is a tuple ${\cal A} = (Q, q_0, F, \Delta)$, where (a) $Q$ is a finite set of {\em states}, (b) $q_0 \in Q$ is the {\em initial state}, (c) $F \subseteq Q$ is the set of {\em final} states, and (d) $\Delta \subseteq Q \times \Sigma \times Q$ is a set of {\em transition rules}. 

The acceptance of a word $w$ by ${\cal A}$ is defined as follows. 
Assuming that $w = a_1 \ldots a_n$ with $a_i \in \Sigma$ for every $i \in \{1, \ldots, n\}$, a {\em run} of ${\cal A}$ on $w$ is a function $\rho : \{0, 1, \ldots, n\} \to Q$ such that $\rho(0) = q_0$ and $(\rho(i), a_{i+1}, \rho(i+1)) \in \Delta$ for every $i \in \{0, \ldots, n-1\}$. Such a run is {\em accepting} if $\rho(n) \in F$. The language accepted by ${\cal A}$, denoted by $L({\cal A})$, is the set of words $w \in \Sigma^*$ such that there exists an accepting run of ${\cal A}$ on $w$.

An NFA ${\cal A} = (Q, q_0, F, \Delta)$ over an alphabet $\Sigma$ is said to be {\em deterministic}, or a deterministic finite automaton (DFA), if for every $q \in Q$ and $a \in \Sigma$, there exists at most one state $q' \in Q$ such that $(q, a, q') \in \Delta$. Every NFA can be translated into a DFA with an exponential blowup that is unavoidable for some families of NFAs. 

Every regular expression $r$ can be translated in polynomial time into an NFA $\mathcal{A}$
such that $L(r)=L(\mathcal{A})$; moreover, if $r$ has length $n$, then $\mathcal{A}$ has at
most $n+1$ states~\cite{EllulSW05}. Conversely, while every NFA can be converted into an
equivalent regular expression, this transformation may incur an unavoidable exponential
blowup. Thus, regular languages admit equivalent descriptions by regular expressions,
NFAs, or DFAs.

\subsection{Double-exponential sample complexity} 

We now establish our main result on uniform generation for finite families of regular
languages. The bounds we obtain do not depend on the chosen representation and should be
viewed as intrinsic to the class of regular languages. We first prove the result for
languages specified by finite automata, and then extend it to regular expressions.
The latter does not follow directly, since translations from automata to regular
expressions can be exponentially more succinct.

\begin{theorem}
\label{thm_regular}
The following statements hold: 
\begin{enumerate}
    \item[a)]
    Let $s \in \mathbb{N}$, and $\mathcal{F}$ be a finite family of infinite  regular languages over an  alphabet of size $a\ge 2$ such that each $L \in \mathcal{F}$ is accepted by an NFA with at most $s$ states. Then $\mathcal{F}$ is $a^{s^{|\mathcal{F}|}}$-generatable. Moreover, the same result holds if each language in $\mathcal{F}$ is defined by a regular expression of length at most $s-1$.

  \item[b)] There exists $k \in \mathbb{N}$ such that, for every $\ell, n \in\mathbb{N}$ with $\ell \geq 1$ and $n \geq 1$, there exists a family $\mathcal{F}$ of infinite binary regular languages such that $|\mathcal{F}| = 2^\ell$, each $L \in\mathcal{F}$ is accepted by a DFA with at most $k \cdot \ell \cdot n$ states, and $\mathcal{F}$ is not $2^{n^{|\mathcal{F}|}}$-generatable. Moreover, the same result holds if each language in $\mathcal{F}$ is defined by a regular expression of length at most  $k \cdot \ell^2 \cdot n^2$.
\end{enumerate}
\end{theorem}
Denoting by $s = k \cdot \ell \cdot n$ the number of states of the DFAs in item b), we get a lower bound of the form $2^{\Omega(s/\log_2|\mathcal{F}|)^{|\mathcal{F}|}}$, which is very close to the upper bound from item a) for the binary alphabet  as long as $s$ is much bigger than $\log_2|\mathcal{F}|$. In fact, for families of size up to $2^{O(s)}$, we still get a double-exponential lower bound in $|\mathcal{F}|$. It is worth pointing out that the number of different binary regular languages recognizable by DFAs with $s$ states is known to be $s^{\Omega(s)}$. It is thus interesting to see if the double-exponential lower bound can be obtained for $\mathcal{F}$ of this size as well.

\begin{proof}[Proof of Theorem \ref{thm_regular}]
We first establish \emph{a)}.  By the standard product construction, each intersection $\bigcap_{L\in\mathcal{S}} L$,
for $\mathcal{S}\subseteq \mathcal{F}$, is recognizable by a finite automaton with at most $s^{|\mathcal{F}|}$ states. If this intersection is finite, it cannot have a word of length $s^{|\mathcal{F}|}$ or larger. Hence, this intersection has size less than $a^{s^{|\mathcal{F}|}}$. By Proposition \ref{prop_char}, the family is $a^{s^{|\mathcal{F}|}}$-generatable. Moreover, we obtain that the same result holds if each language in $\mathcal{F}$ is defined by a regular expression of length at most $s-1$, since every such regular expression can be translated into an NFA with at most $s$ states.
\medskip 

We now proceed to the proof of \emph{b)}. Let $w\in\{0, 1\}^\ell$. Next, let $*<w$ denote the set of all words in $\{0,1\}^\ell$ that precede $w$ in the lexicographic order. Define $L_w^n$ to be the following language. First, this language will consist only of words whose length is a multiple of $\ell$. We thus will think of words in $L_w^n$ as split into blocks of length $\ell$. We put a word into $L_w^n$ if and only if the sequence of its blocks in \emph{odd} positions satisfies the following: there is no consecutive fragment of this sequence  that has more than $n$ blocks $w$ but no block from $*<w$. In turn, the blocks in even position can be filled arbitrarily.

We show that the family $\mathcal{F} = \{L_w^n \mid w\in\{0, 1\}^\ell\}$ satisfies all the requirements. First, $|\mathcal{F}| = 2^\ell$ by construction. It is easy to see that $L_w^n$ is infinite for every $w\in\{0, 1\}^\ell$ (for instance, it contains all words without $w$-blocks).

We now show $L_w^n$ is recognizable by an $O(\ell\cdot n)$-state DFA for every $w\in\{0,1\}^\ell$. This DFA counts the number of $w$-blocks, resetting it to 0 after any block from $*<w$ and checking that this number never exceeds $n$. We thus need a counter, taking $n + O(1)$ possible values. For each value of the counter, we will have an $O(\ell)$-state subautomaton to process blocks in odd positions (remembering the current value of the counter), and $O(\ell \cdot n)$ states overall. Each of these  subautomata tracks for how long the current block coincides with $w$ when we read it from left to right. It suffices to have  $\ell + O(1)$ states for this -- one per prefix of $w$. In the end, we will be able to decide whether the current block is less than $w$ (in which case we reset the counter to 0), equal to $w$ (in which case we increase the counter) or bigger than $w$ (in which case we do nothing with the counter) in the lexicographic order. 

We now show that $\mathcal{F}$ is not $2^{n^{|\mathcal{F}|}}$-generatable. For that we use Proposition \ref{prop_char}, showing that the intersection of languages in the family:
\begin{equation}
    \label{eq_intersection}
    \bigcap\limits_{w\in\{0, 1\}^\ell} L_w^n
\end{equation}
is finite but has size at least $2^{n^{|\mathcal{F}|}}$.

First we show that \eqref{eq_intersection} is finite. Let $w_0 < w_1 < \ldots < w_{2^\ell - 1}$ the words in $\{0, 1\}^\ell$ going in the lexicographic order. If a word belongs to $L^n_{w_0}$, it has at most $n$ odd $w_0$-blocks. If it additionally belongs to $L^n_{w_1}$, then we have at most $n + 1$ ``spaces'' between $w_0$-blocks, and in each space we can put at most $n$ blocks $w_1$. Thus, in total, we can have at most $(n + 1) n$ blocks $w_1$.

Likewise, assuming that we have obtained a finite upper bound on the number of odd blocks for each word in $*< w$, we get the following upper bound on the number of odd $w$-blocks -- at most $n$ times (the sum of bounds for each word in $*< w$ plus 1). Indeed, in each space between a block from $*<w$ we can put at most $n$ blocks $w$, and the number of spaces is the number of corresponding blocks plus 1. This implies that for words in \eqref{eq_intersection} we have an finite upper bound on the number of occurrences of each word in odd blocks, meaning that all words in these intersection have length at most some finite number.

We now show that size of \eqref{eq_intersection} is at least $2^{n^{|\mathcal{F}|}}$. It is enough to construct a word in \eqref{eq_intersection} with at least $n^{|\mathcal{F}|}$ odd blocks (as we can have the same number even blocks and fill them arbitrarily). We construct such word as follows: we put $n$ blocks $w_0$, then in all the spaces we put $n$ blocks $w_1$, and so on. Each time, the number of blocks increases by a factor at least $n$  -- before each old blocks appear $n$ new ones. Hence, after doing this for all $2^\ell$ words in $\{0, 1\}^\ell$, we will obtain at least $n^{2^\ell} = n^{|\mathcal{F}|}$ blocks.

We have shown that b) holds when each language $L^n_w \in \mathcal{F}$ is given by a DFA. To conclude the proof of the theorem, we need to show how to define each language $L^n_w$ with a regular expression of length $O(\ell^2 \cdot n^2)$. Assume that $w = a_1 \cdots a_\ell$, where $a_i \in \{0,1\}$ for every $i \in \{1, \ldots, \ell\}$. Moreover, assume that $Z$ is the set of positions $i \in \{1, \ldots, \ell\}$ such that $a_i = 0$, and $O$ is the of positions $i \in \{1, \ldots, \ell\}$ such that $a_i = 1$. Then for every $i \in O$, let $\textit{prec}_{i}$ be a regular expression that defines the set of words $w' \in \{0,1\}^\ell$ such that, $w'$ precedes $w$ in the lexicographic order because the symbol in position $i$ of $w'$ is $0$, and $w$, $w'$ have the same symbols in positions $1$ to $i-1$. Formally, this regular expression is defined as follows:
\begin{align*}
    \textit{prec}_i \ = \ a_1 \cdots a_{i-1} 0 \underbrace{(0+1) \cdots (0+1)}_{n-i \text{ times}}. 
\end{align*}
Then the following regular expression defines the set of words in $\{0,1\}^\ell$ that precede $w$ in the lexicographic order:
\begin{align*}
    \textit{prec} \ = \ \sum_{i \in O} \textit{\,prec}_i. 
\end{align*}
In the same way, for every $i \in Z$, let $\textit{follow}_{i}$ be a regular expression that defines the set of words $w' \in \{0,1\}^\ell$ such that, $w'$ follows $w$ in the lexicographic order because the symbol in position $i$ of $w'$ is $1$, and $w$, $w'$ have the same symbols in positions $1$ to $i-1$. Formally, this regular expression is defined as follows:
\begin{align*}
    \textit{follow}_i \ = \ a_1 \cdots a_{i-1} 1 \underbrace{(0+1) \cdots (0+1)}_{n-i \text{ times}}. 
\end{align*}
And then the following regular expression defined the set of words in $\{0,1\}^\ell$ that follow $w$ in the lexicographic order:
\begin{align*}
    \textit{follow} \ = \ \sum_{i \in Z} \textit{\,follow}_i. 
\end{align*}
Now for every $i \in \{0, \ldots, n\}$, define $\textit{block}_i$ as a regular expression that defines the sequences of blocks that contain $i$ times the block $w$. Formally, assuming that $r^0 = \varepsilon$ and $r^i = (r r^{i-1})$ for every $i \geq 1$, we have that:
\begin{align*}
\textit{block}_i \ = \ \textit{follow}^* (w \, \textit{follow}^*)^i. 
\end{align*}
With this notation, the following regular expression defines the language $L^n_w$:
\begin{align*}
r^n_w \ = \ \bigg(\bigg(\sum_{i=0}^n \textit{\,block}_i\bigg)\textit{prec}\bigg)^*.
\end{align*}
To conclude the proof, we need to show that the size of $r_w^n$ is $O(\ell^2 \cdot n^2)$. To see why this is the case, we first note that the size of each $\textit{prec}_i$ is $O(\ell)$, and so the size of $\textit{prec}$ is $O(\ell^2)$. Analogously, we have that the size of $\textit{follow}$ is $O(\ell^2)$. On the other side, the size of $\textit{block}_i$ is $O(\ell^2 \cdot i)$, which is $O(\ell^2 \cdot n)$ if $i \leq n$. Hence, the size of $\sum_{i=0}^n \textit{block}_i$ is $O(\ell^2 \cdot n^2)$, from which we deduce that the size of $r_w^n$ is $O(\ell^2 \cdot n^2)$. This concludes the proof of the theorem.
   \end{proof}

\section{Results for LTT languages}


In the previous section, we showed that uniform generation for regular languages can
require a double-exponential number of examples, both when languages are represented by
regular expressions and by finite automata. 
This naturally raises the question of whether imposing additional structure on regular
languages can reduce the sample complexity of generation. In this section, we address
this question by focusing on \emph{locally threshold testable} (LTT) languages, a robust and
expressive subclass of the regular languages admitting several equivalent
characterizations (see, e.g., \cite{DBLP:journals/corr/PlaceRZ13}). We show that restricting attention to LTT languages indeed leads to a
strict reduction in complexity: the sample complexity of uniform generation becomes
single exponential. While this represents a substantial improvement over the
double-exponential bound for general regular languages, the resulting complexity remains
prohibitively large, showing that even strong structural restrictions are insufficient
to make generation feasible.

\subsection{LTT languages}

Given an alphabet $\Sigma$, a word $w' \in \Sigma^*$ is a
\emph{subword} of a word $w \in \Sigma^*$ if there exist $u,v \in \Sigma^*$ such that $w = u w' v$.
If $u=\varepsilon$, then $w'$ is a \emph{prefix} of $w$, and if $v=\varepsilon$, then $w'$
is a \emph{suffix} of $w$.
A \emph{profile} over $\Sigma$ is a tuple
$P = (\mathit{pr}, \mathit{su}, \mathit{In})$, where $\mathit{pr},\mathit{su} \in \Sigma^*$
are respectively a prefix and a suffix, and $\mathit{In}$ is a finite set of constraints
of the form $(w,\theta t)$, with $w \in \Sigma^*$, $\theta \in \{\leq,<,=,>,\geq\}$, and
$t \in \mathbb{N}$ given in unary.\footnote{This means that $t$ is given as a string $a^t$ for some symbol $a \in \Sigma$. For simplicity, we use the notation $t$ instead of $a^t$.}
The language defined by a profile $P$, denoted $L(P)$, consists of all words
$x \in \Sigma^*$ such that $\mathit{pr}$ is a prefix of $x$, $\mathit{su}$ is a suffix of
$x$, and for every $(w,\theta t) \in \mathit{In}$, the number of occurrences of $w$ as a
subword of $x$ satisfies the constraint~$\theta t$.
Moreover, a language $L \subseteq \Sigma^*$ is \emph{locally threshold testable} (LTT) if
there exists a finite sequence of profiles $(P_1,\ldots,P_n)$ such that
$L \;=\; \bigcup_{i=1}^n L(P_i)$. 
\begin{example}\label{exa-LTT}
	{\em Consider the alphabet $\Sigma=\{1,\ldots,n\}$ and the profile
		$P=(\varepsilon,\varepsilon,\mathit{In})$, where
		\[
		\mathit{In}=\{(1,=1),(2,=1),\ldots,(n,=1)\}.
		\]
		Since $\varepsilon$ is both a prefix and a suffix of every word, the profile imposes no
		constraints on the beginning or end of a word. The constraints in $\mathit{In}$ instead
		require that each symbol $i\in\{1,\ldots,n\}$ occur exactly once. Consequently, $L(P)$
		consists precisely of all permutations of the alphabet $\Sigma$, and hence is a finite
		language of size $n!$.
	} \qed
\end{example}

\subsection{Single-exponential sample complexity}

It is easy to see that for every profile $P$, the language $L(P)$ is accepted by an 
NFA: prefix and suffix constraints are local, and each
counting constraint in $\mathit{In}$ can be implemented by a finite-state counter. Since
regular languages are closed under finite union, every LTT language is regular.
However, translating an LTT profile into an equivalent automaton can incur an exponential
blowup. Combining this translation with the general bounds from
Theorem~\ref{thm_regular} therefore yields a triple-exponential upper bound on the sample
complexity when LTTs are given as profiles, and a double-exponential bound when they are
given as NFAs.

We show that these bounds are not inherent to LTT languages. In fact, every finite family
of LTTs admits uniform generation with a \emph{single-exponential} bound, demonstrating
that the structure of LTT profiles can be exploited to obtain substantially more efficient
generation.

\begin{theorem}
	\label{thm_LLT}
	The following statements hold: 
	\begin{enumerate}
		\item[a)]
		Let $s \in \mathbb{N}$, and $\mathcal{F}$ be a finite family of infinite languages over an  alphabet of size $a\ge 2$ such that each $L \in \mathcal{F}$ is defined by an LTT of size at most $s$. Then $\mathcal{F}$ is $a^{3s|\mathcal{F}|((s|\mathcal{F}|)^2+1)^2}
		$-generatable. 
		
		\item[b)] For every $n \in\mathbb{N}$ with $n \geq 4$, there exists a family $\mathcal{F}$ of two infinite languages over an  alphabet of size $n$ such that  each $L \in\mathcal{F}$ is defined by a profile of length at most $7(n+1)$, and $\mathcal{F}$ is not $2^{n}$-generatable. 
	\end{enumerate}
\end{theorem}

\begin{proof}
	We first establish \emph{b)}. 
	The proof builds on the construction in Example~\ref{exa-LTT}.
	We consider a family of two LTT languages, since generation from a single infinite
	language is always trivial, and nontrivial lower bounds necessarily require at least two
	languages.
	
	Let $\Sigma=\{1,\ldots,n\}$ with $n \geq 4$, $P_1=(\varepsilon,\varepsilon,\textit{In}_1)$ be a profile such that $\textit{In}_1=\{(i,\leq 1) : i \text{ is odd}\} \cup \{(i,\geq 1) : i \text{ is even}\}$, and $P_2=(\varepsilon, \varepsilon, \textit{In}_2)$ be a profile such that $\textit{In}_2=\{(i,\geq 1) : i \text{ is odd}\} \cup \{(i,\leq 1) : i \text{ is even}\}$. Moreover, define $\mathcal{F}=\{L(P_1),L(P_2)\}$. Both $L(P_1)$ and $L(P_2)$ are infinite because $\textit{In}_1$ allows arbitrarily many occurrences of even symbols and $\textit{In}_2$ allows arbitrarily many occurrences of odd symbols. On the other hand, the intersection $|L(P_1) \cap L(P_2)| = n!$ since this intersection is defined by the profile $P=(\varepsilon,\varepsilon,\textit{In})$ with $\textit{In}=\{(1,=1)$, $(2,=1)$, $\ldots$, $(n,=1)\}$, and we know from Example \ref{exa-LTT} that $|L(P)| = n!$. Hence, we deduce from Proposition \ref{prop_char} that $\mathcal{F}$ is not $n!-$generatable. As $2^n \leq n!$ for $n \geq 4$, we conclude that $\mathcal{F}$ is not $2^n-$generatable. To finish, notice that the length of both $P_1$ and $P_2$ is bounded by $7(n+1)$.
	
	We now prove \emph{a)}.  Let $\mathcal{F}=\{L_1,\ldots,L_\ell\}$ be a finite family of infinite languages, where each
	$L_i$ is defined by an LTT of size at most $s$. Write each $L_i$ as a finite union of
	profiles $(P^i_1,\ldots,P^i_{m_i})$.
	To show that $\mathcal{F}$ is $a^{3s\ell((s\ell)^2+1)^2}$-generatable, it suffices by
	Proposition~\ref{prop_char} to prove that for every index set
	$I\subseteq\{1,\ldots,\ell\}$ such that $\bigcap_{i\in I} L_i$ is finite,
	\begin{equation}
		\label{eq-upp-ltt}
		\bigg|\bigcap_{i\in I} L_i\bigg| < a^{3s\ell((s\ell)^2+1)^2}.
	\end{equation}

	\noindent\emph{Step 1: Intersections of profiles.}
	The intersection of two profiles of sizes $n_1$ and $n_2$ can be represented by a profile
	of size at most $n_1+n_2$, or is empty. This follows by taking the longer compatible prefix
	and suffix (if they are compatible), and uniting the sets of counting constraints; if
	prefixes or suffixes are incompatible, the intersection is empty.
	By iterating this construction, the intersection of $r$ profiles of size at most $n$ is
	represented by a profile of size at most $rn$.
	
	\medskip
	\noindent\emph{Step 2: Representing finite intersections.}
	By distributivity of intersection over union, the language
	$\bigcap_{i\in I} L_i$ is defined by an LTT consisting of
	$m=\prod_{i\in I} m_i$ 
	profiles, each obtained as the intersection of one profile from each $L_i$.
	Each such profile has size at most $s|I|\le s\ell$.
	Since $\bigcap_{i\in I} L_i$ is finite by assumption, each of these profiles defines a
	finite language.

	\noindent\emph{Step 3: Bounding word length.}
	We now use the following combinatorial fact.
	
	\begin{lemma}
		\label{lema profile}
		If $P$ is a profile of size $r$ such that $L(P)$ is finite, then every word
		$w\in L(P)$ satisfies
		\[
		|w| < 3r(r^2+1)^2.
		\]
	\end{lemma}

	Applying Lemma~\ref{lema profile} with $r=s\ell$, we conclude that every word
	$w\in\bigcap_{i\in I} L_i$ has length strictly less than
	$3s\ell((s\ell)^2+1)^2$. The proof of Lemma \ref{lema profile} can be found in the appendix. 
	
	\medskip
	\noindent\emph{Step 4: Counting words.}
	The number of words over an alphabet of size $a\ge 2$ of length strictly less than
	$3s\ell((s\ell)^2+1)^2$ is bounded by
	$a^{3s\ell((s\ell)^2+1)^2}$.
	Therefore~\eqref{eq-upp-ltt} holds, completing the proof of~(a).
	
	\medskip
It remains to establish Lemma \ref{lema profile}.

\begin{proof}[Proof of Lemma \ref{lema profile}]
	We require the following ``pumping lemma'' for profiles
	\begin{lemma}
		\label{lema ltt}
		Let $w_1,\ldots,w_m \in \Sigma^*$ with $m \geq 1$, $\ell=\max\{|w_i| \mid 1\leq i \leq m\}$, and $w \in \Sigma^*$ such that $|w| \geq 2(\ell m + 1)^{2}$ and $w$ does not have any of $w_1,\ldots,w_m$ as subword. Then there exist $u,x,v \in \Sigma^*$ such that $w=uxv$, $|x|>\ell$ and the word $uxxv$ also does not have any of $w_1,\ldots,w_m$ as subword.
	\end{lemma}
	
	\begin{proof}
		Let $w \in \Sigma^*$ be a word that satisfies the hypothesis of the lemma. Moreover, let $\pref$ be the set of prefixes of the words $w_1,\ldots,w_m$. Notice that $|\pref| \leq \ell m + 1$.
		
		Assume that $w = w[1] \cdots w[n]$, where $w[i] \in \Sigma$ for every $i \in \{1, \ldots, n\}$. Then define an assignment $\tau :\{1, \ldots, n\} \to \pref$ as follows.
		Take position $i \in \{1,\ldots,n\}$, and let $\pref_i$ be the set of words of the form $w[i-k+1]\cdots w[i]$ such that $k \in \mathbb{N}$ and $w[i-k+1]\cdots w[i] \in \pref$. Observe that $\pref_i \neq \varnothing$ because $\varepsilon \in \pref_i$, and it is a finite set since $|\pref_i| \leq |\pref| \leq \ell m + 1$. Then
		we define $\tau(i)$ as the word in $\pref_i$ with maximum length.
		
		Given that $|w| \geq 2(\ell m+1)^2$ and $|\pref| \leq \ell m$ + 1, there exists an element $p$ in $\pref$ that is assigned to at least $2(\ell m + 1)$ positions (that is, $|\{i \in \{1, \ldots, n\} \mid \tau(i) = p\}| \geq 2(\ell m + 1)$). Let $i_0$ be the minimum position such that $\tau(i_0) =p$, and $i_1$ be the maximum position such that $\tau(i_1) = p$. Then define $x$ as the subword of $w$ between positions $i_0$ and $i_1$, including $w[i_0]$ but excluding $w[i_1]$. Moreover, define $u$ and $v$ as the remaining prefix and suffix of $w$, so that $w = u x v$. This word is depicted in the following figure, where we also show the first and last occurrences of $p$:
		
		\begin{center}
			\includegraphics[width=0.8\linewidth]{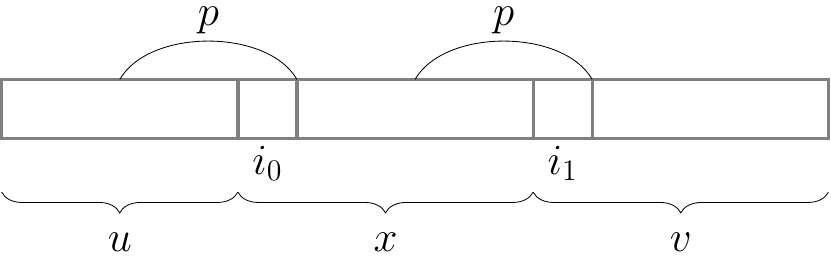}
		\end{center}
		
		We have that $|x| > \ell$ because $x$ contains at least $2(\ell m + 1)$ positions. Hence, it remains to prove that $uxxv$ does not have any of $w_1,\ldots,w_m$ as subword. The word $uxxv$ is depicted in the following figure:

		\begin{center}
			\includegraphics[width=0.8\linewidth]{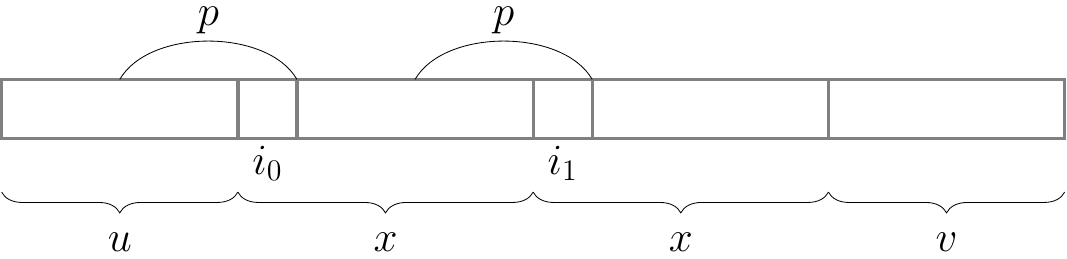}
		\end{center}
	
	Suppose, by contradiction, that there exists $i \in \{1, \ldots, m\}$ such that $w_i$ is a subword of $uxxv$. Given that $w_i$ is not a subword of $w$, it follows that $w_i$ is a subword of $xx$, with a nonempty part lying in the first occurrence of $x$ and a nonempty part lying in the second occurrence of $x$ (as otherwise $w_i$ would be a subword of $ux$ or $xv$, which contradicts our assumption that $w_i$ is not a subword of $w = uxv$).
	This gives rise to two cases for our proof, depending on how $p$ and $w_i$ are related in $uxxv$.
	
	\begin{enumerate}
		\item Assume first that the first symbol of $w_i$ occurs before the first symbol of $p$, as depicted in the following figure:

			\begin{center}
				\includegraphics[width=0.8\linewidth]{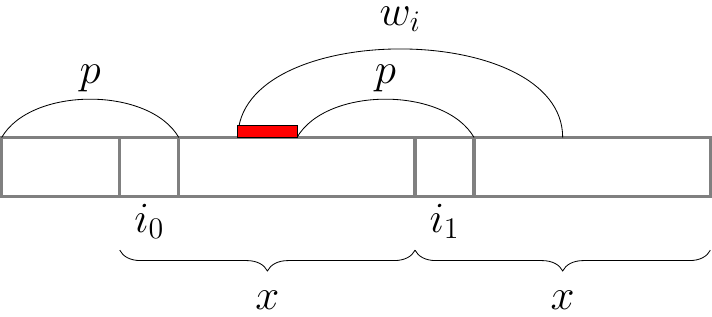}
			\end{center}
			In particular, the red region in the figure is not empty. Notice that in this case, $p$ is a subword of $w_i$, since only the last symbol of $p$ appears in the second occurrence of $x$ (at position $i_1$), and a nonempty part of $w_i$ lies in the second occurrence of $x$.
			Then there exists a prefix $y$ of $w_i$ such that $p$ is a proper suffix of $y$. But this leads to a contradiction since $y \in \pref_{i_1}$, $|p| < |y|$, and~$\tau(i_1) = p$.
			
			\item
			Assume now that the first symbol of $w_i$ does not occur before the first symbol of $p$, as depicted in the following figure:

				\begin{center}
					\includegraphics[width=0.8\linewidth]{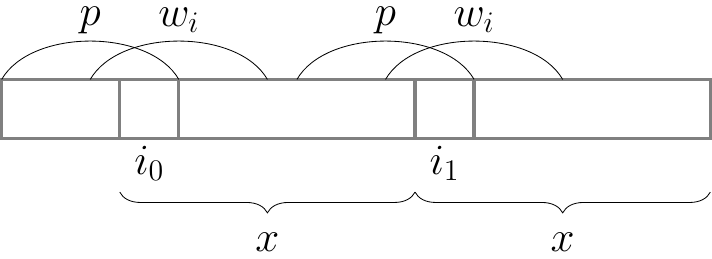}
				\end{center}    
			Then given that the word $x$ occurs twice in $uxxw$, the same words appear at the beginning and at the end of the $uxxw$, as depicted in the figure. But this implies that $w_i$ is a subword of $ux$ since $|w_i| \leq \ell$ and $|x| > \ell$. We conclude that $w_i$ is a subword of $w = uxv$, which contradicts the hypothesis of the lemma.
		\end{enumerate}
		This finishes the proof of the lemma. 
	\end{proof}
	
	We now finish the proof of Lemma \ref{lema profile}.
	Let $P=(\textit{pr},\textit{su},\textit{In})$ be a profile over an alphabet $\Sigma$ that defines a finite language, where $\textit{In}=\{(w_1,\theta_1t_1),\ldots,(w_m,\theta_mt_m)\}$. Moreover, assume that $k \in \{0, \ldots, m\}$ satisfies the following:
	\begin{itemize}
		\item $\theta_i \in \{<, \leq, =\}$ for every $i \in \{1, \ldots, k\}$, and
		
		\item $\theta_i \in \{>, \geq\}$ for every $i \in \{k + 1, \ldots, m\}$.
	\end{itemize}
	Notice that if $k = 0$, then every $\theta_i$ is either $>$ or $\geq$, and if $k = m$, then every $\theta_i$ is either $<$ or $\leq$ or $=$. 
	
	If $m = 0$ or $k = 0$, then $L(P)$ is infinite, and we obtain a contradiction as we assume that $L(P)$ is finite. Hence, we have that $m \geq 1$ and $k \in \{1, \ldots, m\}$, and we define $\ell = \max \{ |w_i| \mid 1 \leq i \leq k\}$. Moreover, we assume that $s$ is the length of profile $P$.
	
	If $L(P) = \emptyset$, then the lemma trivially holds. Hence, assume that $L(P) \neq \emptyset$, and let $w$ be a word of $L(P)$ of maximum length (such a word exists since $L(P)$ is finite). For the sake of contradiction, assume that $|w| \geq 3 (s^2 + 1)^2 s$. Then, considering that:
	\begin{align*}
		|\textit{pr}| + |\textit{su}| + 
		\bigg(\sum_{i=1}^m |w_i| (|t_i| + 1)\bigg) \ &\leq \ s^2,\\
		2(\ell k + 1)^2 \cdot \bigg(1+m+\sum_{i=1}^m |t_i|\bigg) \ &\leq \ 2(s^2 + 1)^2 s,\\
		s^2 + 2(s^2 + 1)^2 s \ &\leq \ 3 (s^2 + 1)^2 s,
	\end{align*}
	we conclude that:
	\begin{multline}
		\label{eq-pol-prof}
		|\textit{pr}| + |\textit{su}| + \bigg(\sum_{i=1}^m |w_i| (|t_i| + 1)\bigg) \ + \\ 2(\ell k + 1)^2 \cdot \bigg(1+m+\sum_{i=1}^m |t_i|\bigg) \ \leq \ |w|.
	\end{multline}
	Therefore, we conclude that there exists $u, w', v \in \Sigma^*$ such that:
	\begin{itemize}
		\item $w = u w' v$,
		
		\item $\textit{pr}$ is a prefix of $u$ and 
		$\textit{su}$ is a suffix of $v$,
		
		\item $|w'| \geq 2(\ell k + 1)^2$,
		
		\item $w'$ does not have any of $w_1$, $\ldots$, $w_k$ as subword,
		
		\item the number of occurrences of $w_i$ as a subword of $u$ or $v$ is at least $t_i + 1$, for every $i \in \{k+1, \ldots, m\}$ such that $\theta_i = \ >$, and
		
		\item the number of occurrences of $w_i$ as a subword of $u$ or $v$ is at least $t_i$, for every $i \in \{k+1, \ldots, m\}$ such that~$\theta_i = \ \geq$.
	\end{itemize}
	By lemma \ref{lema ltt}, there exist $u', x, v' \in \Sigma^*$ such that $w' = u'xv'$, $|x| > \ell$ and the word $u'xxv'$ does not have any of $w_1$, $\ldots$, $w_k$ as subword. Hence, we conclude that $uu'xxv'v \in L(P)$ by the previous conditions. But then we obtain a contradiction since $|uu'xxv'v| > |w|$ (given that $w = uw'v = uu'xv'v$ and $|x| > 0$) and $w$ is an element of $L(P)$ of maximum length.
\end{proof}
	
\end{proof}

\section{Results for non-erasing pattern languages}

Let $\mathcal{X}$ be a countably infinite set of \emph{variables}, and use uppercase
letters $X,Y,\ldots$ to denote its elements. A \emph{pattern} over a finite alphabet
$\Sigma$, assumed disjoint from $\mathcal{X}$, is a word in $(\Sigma \cup \mathcal{X})^*$.

A \emph{valuation} over $\Sigma$ is a mapping $\mu : \mathcal{X} \to \Sigma^*$ assigning to
each variable $X \in \mathcal{X}$ a word $\mu(X) \in \Sigma^*$. The valuation is
\emph{non-erasing} if $\mu(X)$ is a nonempty word 
for every $X \in \cal X$. If $p$ is a pattern and $\mu$ a
valuation, we write $\mu(p)$ for the word in $\Sigma^*$ obtained by simultaneously
replacing each occurrence of a variable $X$ in $p$ with the word $\mu(X)$.

Each pattern $p$ over $\Sigma$ defines a \emph{non-erasing pattern language} over $\Sigma$
by
\[
L(p) \;=\; \{\mu(p) \mid \mu \text{ is a non-erasing valuation over } \Sigma\}.
\]
As shown in the next example, non-erasing pattern languages are incomparable with
context-free, regular, and LTT languages.

\begin{example}
{\em Consider the pattern $p = XX$. Then $L(p)$ consists of all words of the form $ww$, where
$w$ is a nonempty word over $\Sigma$. This language is not context-free, and hence not
regular. In turn, the regular language $a^*b^*$, which is even LTT, is not a non-erasing
pattern language.}
\qed
\end{example}

Non-erasing pattern languages are known to have favorable properties for language
identification in the limit under Gold’s model \cite{DBLP:journals/iandc/Angluin80}. We show next that, despite this, finite
families of such languages need not admit feasible language generation in the limit.

\begin{theorem}
\label{thm_patterns}
For every $n \in \mathbb{N}$, there exists a family $\{p_1,p_2\}$ of two 
infinite non-erasing pattern languages over the alphabet $\Sigma = \{a,b\}$ satisfying the following: $p_1$ is defined by a pattern of size $44n$, $p_2$ is defined by a pattern of size 2, and $\{p_1,p_2\}$ is not 
$2^n$-generatable.
\end{theorem}

\begin{proof}
By Proposition~\ref{prop_char}, it suffices to construct patterns $p_1,p_2$ such that
$L(p_1)$ and $L(p_2)$ are infinite and
\(
|L(p_1)\cap L(p_2)| = 2^n.
\)
Consider the following patterns of size 22:
\begin{eqnarray*}
q \ &= \ &XaXbXaabbabaXbabaabbab,\\
r \ &= \ &abaabbabaXbabaabbXaXbX.
\end{eqnarray*}
It has been shown by \citeauthor{DBLP:journals/ijfcs/NowotkaS18} that
\(
|L(q)\cap L(r)|=2.
\)
For each $i\in[n]$, let $q_i$ be the pattern obtained from $q$ by replacing every
occurrence of $X$ with a fresh variable $X_i$; define $r_i$ analogously (use the same variable $X_i$ in $r_i$ as in $q_i$). Now set
$$
p_1 \ = \ q_1 q_2 \cdots q_n r_1 r_2 \cdots r_n, \ \ \ \ \ \ 
p_2 \ = \ YY.
$$
Clearly, $L(p_1)$ and $L(p_2)$ are infinite. Moreover, since $p_2$ forces words of the
form $ww$ and the two halves of $p_1$ have the same length (because $|q|=|r|$), any word
in $L(p_1)\cap L(p_2)$ must be of the form
\[
w w
\qquad\text{with}\qquad
w \in L(q_1 q_2\cdots q_n)\cap L(r_1 r_2\cdots r_n).
\]
Thus, $w w \in L(p_1)\cap L(p_2)$ iff there exists a non-erasing valuation $\mu$ such that
$\mu(q_1 q_2\cdots q_n)=\mu(r_1 r_2\cdots r_n)$. 
Since $|q_i|=|r_i|$ for each $i$, equality of the concatenations implies blockwise
equality:
$\mu(q_i)=\mu(r_i)$ 
for all $i\in[n]$. 
Because $q_i$ and $r_i$ share \emph{only} the variable $X_i$, the condition
$\mu(q_i)=\mu(r_i)$ depends only on the value of $\mu(X_i)$, and the choices for
different indices $i$ are independent. Furthermore, by the result of
\citeauthor{DBLP:journals/ijfcs/NowotkaS18}, for each $i$ there are exactly two
non-erasing valuations for $X_i$ that satisfy $\mu(q_i)=\mu(r_i)$, i.e.,
\(
|L(q_i)\cap L(r_i)|=2.
\)
It follows that there are exactly $2^n$ choices for $(\mu(X_1),\ldots,\mu(X_n))$ that
simultaneously satisfy $\mu(q_i)=\mu(r_i)$ for all $i$, and each such choice yields a
distinct word in $L(p_1)\cap L(p_2)$ (already the $i$th block differs whenever the choice
for $X_i$ differs). Hence,
$|L(p_1)\cap L(p_2)| \;=\; \prod_{i=1}^n |L(q_i)\cap L(r_i)| \;=\; 2^n$. 
By Proposition~\ref{prop_char}, 
the family $\{p_1,p_2\}$ is not
$2^n$-generatable.

Finally, $p_2$ has size 2, and $p_1$ has size $44 n$ as it is a concatenation of $2n$ patterns of size 22.
\end{proof} 

\section{Final Remarks}

\paragraph{Implications for learning}
Kleinberg and Mullainathan show that language generation is possible in principle even
under worst-case presentations of positive data, without relying on statistical
assumptions~\cite{kleinberg2024language}. Our results refine this picture by showing that
the feasibility of generation depends critically on the structure of the language family.
In particular, large finite intersections among languages can force the number of
required examples to be enormous or even non-computable. At the same time, this suggests
why generation may be tractable in practice: if such intersections are small for the
language structures encountered by models, efficient generation may still be possible.
The key question, therefore, is not only whether generation is possible, but under which
structural conditions it can be achieved efficiently.

\paragraph{Connections to formal language theory.}
Beyond their implications for learning theory, our results also connect to classical
problems in formal language theory. In particular, the lower bounds we obtain are
governed by the size of the largest finite intersection of languages in the family (see
Proposition~\ref{prop_char}). While this quantity has received little direct attention,
its analysis relies on techniques similar to those used in studying the complexity of
deciding whether the intersection of a finite 
collection of languages is nonempty
\cite{DBLP:conf/focs/Kozen77}. From this perspective, language generation in the limit
offers a new lens on the complexity of language intersection.

	\section*{Acknowledgements} 

Barcelo and Kozachinskiy are funded by the National Center for Artificial
Intelligence CENIA FB210017, Basal ANID. Barcelo is also funded by ANID Millennium Science Initiative Program Code ICN17002. Kozachinskiy is further supported by ANID Fondecyt Iniciacion grant
11250060. Cofr\'e is supported by the Faculty of Mathematics at Pontificia Universidad Cat\'olica de Chile through the master's scholarship BMA.


\newpage

\appendix

\end{document}